\RequirePackage{fix-cm}
\RequirePackage{amsmath}

\documentclass{article}
\usepackage[T1]								{fontenc} 							
\usepackage[ansinew]						{inputenc} 							
\usepackage									{ellipsis} 							
\usepackage									{lmodern} 							
\usepackage									{amssymb}							
\usepackage									{bm, bbm}
\usepackage[fixamsmath]						{mathtools} 						
\usepackage[shortlabels]					{enumitem} 							
\usepackage[spacing,kerning]				{microtype}
\usepackage									{geometry}
\usepackage									{csquotes}						
\usepackage									{graphicx}

\usepackage{amsmath,mathrsfs}
\usepackage{amssymb}
\usepackage{color}
\usepackage{dsfont}
\usepackage{tikz}
\usepackage{pgfplots}
\usepgfplotslibrary{fillbetween}
\usepackage{pdfsync}
\usepackage[font=footnotesize]{caption}

\newcommand\NoBlackBoxes{\global\overfullrule0pt}
\NoBlackBoxes

\newcommand{\N}{\mathbb{N}}

\makeatletter
\let\serieslogo@\relax
\let\@setcopyright\relax

\newtheorem{definition}{Definition}[section]
\newtheorem{theorem}[definition]{Theorem}

\newtheorem{proposition}[definition]{Proposition}

\newtheorem{remark}[definition]{Remarks}

\newtheorem{situation}[definition]{Situation}
\newtheorem{example}[definition]{Example}
\newtheorem{proof}[definition]{Proof}

\newcommand{\R}{{\mathbb{R}}}

\renewcommand{\epsilon}{\varepsilon}
\renewcommand{\phi}{\varphi}
\newcommand{\vep}{\varepsilon}

\usepackage{ulem}

\begin{document}

\setcounter{page}{1}

\title{Some Remarks on Replicated Simulated Annealing}



\author{
  \begin{tabular}{ccc}
    Vincent Gripon & Matthias L\"owe & Franck Vermet \\
    \footnotesize{UMR CNRS Lab-STICC} & \footnotesize{Fachbereich Mathematik und Informatik} & \footnotesize{UMR CNRS 6205}\\
    \footnotesize{IMT Atlantique} & \footnotesize{Universit\"at M\"unster} & \footnotesize{Universit\'e de Bretagne Occidentale}\\
  \end{tabular}
}

%

%

%
%


\date{}

\maketitle

%
%

\begin{abstract}
  Recently authors have introduced the idea of training discrete weights neural networks using a mix between classical simulated annealing and a replica ansatz known from the statistical physics literature. Among other points, they claim their method is able to find robust configurations. In this paper, we analyze this so called ``replicated simulated annealing'' algorithm. In particular, we give criteria to guarantee its convergence, and study when it successfully samples from configurations. We also perform experiments using synthetic and real data bases.
\end{abstract}



\section{Introduction}
In the past few years, there has been a growing interest in finding methods to train discrete weights neural networks. As a matter of fact, when it comes to implementations, discrete weights allow to reach a better efficiency, as they considerably simplify the multiply-accumulate operations, with the extreme case where weights become binary and there is no need to perform any multiplication anymore. Unfortunately, training discrete weights neural networks is complex in practice, since it basically boils down to a NP-hard optimization problem. To circumvent this difficulty, many works have introduced techniques that aim at finding reasonable approximations~\cite{courbariaux2015binaryconnect,choi2018pact,jain2019trained,esser2019learned}.

Among these works, in a recent paper Baldassi et al. \cite{Baldassi_etal} discuss the learning process in artificial neural networks with discrete weights and try to explain why these networks work so efficiently.
Their approach is based on an analysis of the learning procedure in
artificial neural networks. In this process a
huge number of connection weights are adjusted using some stochastic optimization algorithm for a given target function. Some of the resulting optima of the target or energy function have better computational performance and
generalization properties, other worse.
The authors in \cite{Baldassi_etal} propose that the better and more robust configurations of weights lie in dense regions with many maxima or minima (depending on the sign) of the target function, while the optimal configurations that are isolated, i.e.\ far away from the next optimum of the energy function have poor computational performance.
They propose a new measure, called the robust ensemble, that suppresses such configurations with bad computational performance. 
On the other hand, the robust ensemble amplifies the dense regions with many good configurations. In \cite{Baldassi_etal} the authors present various algorithms to sample from this robust ensemble, one of them is {\it Replicated Simulated Annealing} or {\it Simulated Annealing with Scoping}. This algorithm combines the replica approach from statistical physics with the simulated annealing algorithm that is supposed to find the minima (or maxima) of a target function.
The replica technique is used to regularize the highly non-convex target or energy function (another regularization idea was introduced recently in \cite{Biroli_2020}), while the simulated annealing algorithm is used afterwards to minimize this new energy.
We will define Replicated Simulated Annealing in Section 2.

To give a first impression of this algorithm, assume we have $\Sigma:=\{-1,+1\}^N$ as our state space and $N$ is large. On $\Sigma$ we have very rugged energy function $E:\Sigma \to \R$ and assume that $\min_{\sigma \in \Sigma}E(\sigma)=0$. To find the minima of $E$ one could run a Metropolis algorithm for the Gibbs measure at inverse temperature $\beta >0$:
$$\pi_\beta(\sigma):= \frac{\exp(-\beta E(\sigma))}{Z_\beta}, \qquad \sigma \in \Sigma.
$$
Here $Z_\beta:= \sum_{\sigma'} \exp(-\beta E(\sigma'))$ is the partition function of the model, a normalizing factor that makes $\pi_\beta$ a probability measure. If one carefully lowers the temperature, i.e.\ if one increases $\beta$ slowly enough during the process the corresponding Markov chain will get stuck in one of the maxima of $\pi_\infty$ which are easily seen to be the minima of $E$. This is the classical Simulated Annealing algorithm, cf.\ \cite{KirkGelaVecc83} or \cite{Geman_Geman} for the seminal papers. The question how to choose the optimal dependence of $\beta$ from time $t$ and its convergence properties have  been extensively discussed. We just mention \cite{hajek}, \cite{HolleyStroock}, \cite{GoetzeSA}, \cite{DeuschelMazza}, \cite{catoni}, for a short and by far not complete list of references. The upshot is that good a ''cooling schedule'' is of the form $\beta_t=\frac 1m \log (1+ C t)$, where $m$ can be roughly described as the largest hill to be climbed to get from an arbitrary state to one of the global minima of the target function.

However, sometimes not all the global minima are equally important, in particular one may be interested in
regions with many global minima (or almost global minima), so called dense regions. An obstacle may be, that
$E$ exhibits many global minima, but only relatively few of them are in dense regions. Let us motivate this question by
a central example which we will often have in mind in this context and which also is one of the central
objects in \cite{Baldassi_etal}.

\begin{example}\label{perceptron}
Assume we have patterns $\xi^1, \ldots, \xi^M$, $\xi^\mu \in \{\pm 1\}^N, \mu=1, \ldots, M$ where $M=\alpha N$ for some $\alpha>0$. Each of these patterns belongs to one of two groups, which we indicate by
$\vartheta(\xi^\mu)= \vartheta^\mu \in \{\pm 1\}.$

The task is to classify these patterns.
One of the standard methods to do this in machine learning is the perceptron.
Perceptron was defined in the 1960s by Rosenblatt \cite{Ros}.
It is one of the first mathematical models of neural networks, inspired by the biological phenomenon of vision.
In its simplest form, the one we are studying here, it is a single neuron, corresponding to the mathematical model of Mac Culloch and Pitts \cite{MCP}. It is then a binary classifier, which separates two sets of points linearly separable by a hyperplane.
In mathematical terms it maps its input $x$ to $f(x)$ which is either $1$ or $0$, and thus puts it into one of two classes.
This decision is made with the help of a vector of weights
$$W=(W_1, \ldots, W_N) \in \{-1,+1\}^N.$$
More precisely,
$$
f(x)=\begin{cases}1&{\text{if }}\  \langle {W}, x \rangle >0,\\0&{\text{otherwise}}\end{cases}
$$
where $\langle W,x\rangle $ is the dot product $\langle W,x\rangle=\sum_i W_i x_i$.
These weights
$W=(W_1, \ldots, W_N) \in \{-1,+1\}^N$ have to learned and we want the
classification to be perfect, i.e. we want that
$$
\Theta (\vartheta^\mu \langle W, \xi^\mu \rangle):= \Theta (\vartheta^\mu \sum_{i=1}^N W_i \xi_i^\mu)=1
$$
for all $\mu=1, \ldots M$. Here
$$\Theta(x)=\left\{\begin{array}{ll} 1 & \mbox{if } x>0\\ 0 & \mbox{otherwise}\end{array} \right.
$$
denotes the Heaviside-function. Hence our classification task is fulfilled if
$$
\sum_{\mu=1}^M  \Theta (-\vartheta^\mu \langle W, \xi^\mu \rangle)=0
$$
(where we assume that $N$ is even to avoid the specification of tie-breaking rules) or, equivalently
\begin{equation}\label{classify}
\prod_{\mu=1}^M  \Theta (\vartheta^\mu \langle W, \xi^\mu \rangle)=1.
\end{equation}
Note that in Rosenblatt's initial model, the weights $(W_1,\ldots, W_N)$, called synaptic weights, are real-valued and not restricted to take their values in $\{-1,+1\}^N$.
The objective now is to find weights $W$ such that \eqref{classify} is true, i.e. we are searching for weights $W$ such
 that $\overline{E}(W)=-\prod_{\mu=1}^M  \Theta (\vartheta^\mu \langle W, \xi^\mu \rangle)$ is minimal.
Obviously, this optimization problem is of the above mentioned form. However, one prefers weights $W$ in
co called dense regions, i.e.\ weights that are surrounded by weights $\tilde W$ that are also minima of
$\overline{E}$. The idea is that these states have good generalization properties or a small generalization
error. This means that we want to find weights, that still classify input patterns correctly, which we have
not seen in our training set $\xi^1, \ldots, \xi^M$. It is at least plausible  that weights with
a small generalization error lie in dense regions of $\{\pm1\}^N$.
\end{example}

In this work we are interested in making explicit convergence properties of the algorithm of Replicated Simulated Annealing. We also perform experiments using synthetic and real data bases. The outline is as follows: in Section 2 we mathematically formalize and describe the algorithm of Replicated Simulated Annealing, in Section 3 we study its convergence properties. Naturally, this convergence will be studied on an infinite time horizon. This is a slightly different set-up than in \cite{Baldassi_etal}, where experiments are performed for finite time. However, the question, whether or not Replicated Simulated Annealing converges is the first question that should be analyzed, before studying which choice of parameters yields the best results.
This latter question is addressed in Section 4, where we perform experiments using synthetic and real data bases.
Of course, the time horizon for such experiments is finite. On the other hand, in finite time we can control the influence of the choice of the parameters on the performance which is hard to control theoretically.
Finally, Section 5 is a conclusion.

\section{Replicated Simulated Annealing}
Recall that we are searching for the minima of a function $E:\Sigma \to \R$, where $\Sigma:=\{-1,+1\}^N$.
To find minima of $E$ in dense regions of the state space the authors in \cite{Baldassi_etal} propose a new measure given by
\begin{equation}
P_{y,\beta,\gamma}(\sigma):= \frac{\exp(y \Phi_{\beta,\gamma}(\sigma))}{Z(y,\beta, \gamma)}
\end{equation}
where
\begin{equation}\label{eq:part_func_gen}
Z(y,\beta, \gamma):= \sum_{\sigma''}\exp(y \Phi_{\beta,\gamma}(\sigma'')).
\end{equation}
$P_{y,\beta,\gamma}$ has, at least formally, the structure of a Gibbs measure at inverse temperature $y$. Its ''energy function'' is given by
\begin{equation}
\Phi_{\beta,\gamma}(\sigma):= \log \sum_{\sigma' \in \Sigma} \exp(-\beta E(\sigma')-\gamma d(\sigma, \sigma')),
\end{equation}
where $d(\cdot, \cdot)$ is some monotonically increasing function of a distance on $\Sigma$. This distance will be chosen below, but there are not too many reasonable essentially different distance functions on $\Sigma$, anyway.

Since $\Phi_{y,\beta,\gamma}$ weights each configuration $\sigma$ by a function of its distance to $\sigma'$ and the $\sigma'$ again by an exponential of their energy, it is, indeed, plausible that $\Phi_{y,\beta,\gamma}$ is much smoother than $E$ and will have its minima in dense regions. We will come back to this question in the next section.

However, a serious problem is, how one could simulate from the measure $P_{y,\beta,\gamma}$. Indeed, computing the ''energy'' $\Phi_{y,\beta,\gamma}(\sigma)$ of a single configuration $\sigma$ involves, among others, computing $E(\sigma')$ for all $\sigma' \in \Sigma$.
Computing these values is almost as hard as finding the minima of $E$ (even though one might not be immediately able to tell which of these minima are in dense regions). To find a promising algorithm that does not rely on computing all the values of $E(\sigma)$, Baldassi et.\ al.\ \cite{Baldassi_etal} propose the following:

First of all assume that $y \ge 2$ is an integer. Second take as function of the distance between two spins $\sigma$ and $\sigma'$ the (negative) inner product: $d(\sigma, \sigma')=-\langle \sigma, \sigma'\rangle$. As a matter of fact, this is a natural choice, since two natural distance functions, the Hamming distance and the square of the Euclidian distance are functions of the inner product:
$
\sum_{i} (\sigma_i-\sigma'_i)^2=2N-2 \langle \sigma, \sigma'\rangle$ 
as well as 
$d_H(\sigma, \sigma')= \frac{N -\langle \sigma, \sigma'\rangle}2
$
and the $N$ dependent terms cancel, because they also occur in $Z(y,\beta, \gamma)$.
Using the fact that $y$ is an integer, we can now compute the partition function $Z(y,\beta, \gamma)$
(by replacing $\sigma''$ by $\sigma$ in \eqref{eq:part_func_gen}) of this model:
\begin{eqnarray*}
&&Z(y,\beta, \gamma) = \sum_{\sigma \in \Sigma} \exp(y \Phi_{\beta, \gamma}(\sigma))
= \sum_{\sigma \in \Sigma} \exp(\sum_{a=1}^y \Phi_{\beta, \gamma}(\sigma))\\
&=&  \sum_{\sigma \in \Sigma} \prod_{a=1}^y  \sum_{\sigma^a \in \Sigma} \exp(-\beta E(\sigma^a)-\gamma d(\sigma, \sigma^a))\\
&=&\sum_{\sigma \in \Sigma}\sum_{\sigma^1 \in \Sigma}\ldots \sum_{\sigma^y \in \Sigma} \exp(-\beta E(\sigma^1)-\gamma d(\sigma, \sigma^1))\ldots \exp(-\beta E(\sigma^y)-\gamma d(\sigma, \sigma^y))\\
&=&\sum_{\sigma \in \Sigma}\sum_{\{\sigma^a\} \in \Sigma^y}\exp(-\beta \sum_{a=1}^y E(\sigma^a)-\gamma \sum_{a=1}^y d(\sigma, \sigma^a)).
\end{eqnarray*}
Here $\sum_{\{\sigma^a\}}$ is the sum over all $\sigma^1, \ldots, \sigma^y$.
Hence $Z(y,\beta, \gamma)$ can be considered as a partition function on the space of all $(\sigma, \{\sigma^a\})\in \Sigma^{y+1}$ of the measure
\begin{equation}\label{eq:Q}
Q((\sigma, \{\sigma^a\}):= \frac{\exp(-\beta \sum_{a=1}^y E(\sigma^a)-\gamma \sum_{a=1}^y d(\sigma, \sigma^a))}{Z(y,\beta,\gamma)}.
\end{equation}
Its marginal with respect to the second coordinate $\{\sigma^a\}$ is given by
\begin{equation}\label{eq:Qbar}
\overline{Q}(\{\sigma^a\}):= \frac{\sum_{\sigma} \exp(-\beta \sum_{a=1}^y E(\sigma^a)-\gamma \sum_{a=1}^y d(\sigma, \sigma^a))}{Z(y,\beta,\gamma)}.
\end{equation}
Making use of our choice $d(\sigma, \sigma')=-\langle \sigma, \sigma'\rangle$ we obtain for the numerator in $\overline{Q}$:
\begin{eqnarray*}
{Z(y,\beta,\gamma)}\overline{Q}(\{\sigma^a\})&=& \sum_{\sigma} \exp\left(-\beta \sum_{a=1}^y E(\sigma^a)+\gamma \sum_{a=1}^y
\langle \sigma, \sigma^a\rangle\right)\\
&=& \frac{2^N}{2^N}\sum_{\sigma} \exp\left(-\beta \sum_{a=1}^y E(\sigma^a)+\gamma \sum_{i=1}^N\sum_{a=1}^y
\sigma_i \sigma^a_i\right)\\
&=& 2^N \exp\left(-\beta \sum_{a=1}^y E(\sigma^a)+\sum_{i=1}^N \log\cosh(\gamma \sum_{a=1}^y
\sigma_i^a)\right)
\end{eqnarray*}
Putting the $2^N$ into the normalizing constant we thus obtain that
$$
\overline{Q}(\{\sigma^a\})= \frac{\exp\left(-\beta \sum_{a=1}^y E(\sigma^a)+\sum_{i=1}^N
\log\cosh(\gamma \sum_{a=1}^y \sigma_i^a)\right)}{Z'(y,\beta,\gamma)}.
$$
This form of the measure $\overline{Q}$ is now accessible to a Simulated Annealing algorithm: being in $\{\sigma^a\}$ one picks one of the
$\sigma^a$ at random and one coordinate $\sigma_i^a$ of $\sigma^a$ at random and flips it to become $-\sigma_i^a$. This new configuration is then accepted with the usual Simulated Annealing probabilities.
\begin{example} (Example \ref{perceptron} continued)
In our perceptron example we so far proposed the energy function
$$
\overline{E}(W)=-\prod_{\mu=1}^M  \Theta (\vartheta^\mu \langle W, \xi^\mu \rangle)=-\prod_{\mu=1}^{\alpha N}  \Theta (\vartheta^\mu \langle W, \xi^\mu \rangle).
$$
This function, however, may be a bit unwieldy when using Simulated Annealing, since it just tells how many patterns have been classified correctly but not whether we are moving in a ''good'' or a ''bad'' direction when the proposed configuration $\{{W'}^a\}$ has the same energy as the old configuration $\{{W}^a\}$. We therefore propose (as e.g. \cite{Baldassi_etal}) to use the energy function
$$
E(W)= \sum_{\mu=1}^{\alpha N} E^\mu(W) \quad \mbox{with  } E^\mu(W)= R(-\vartheta^\mu\langle W, \xi^\mu\rangle),
$$
instead.
Here $R(x)= \frac{x+1}2 \Theta(x)$ and we again assume that $N$ is odd, otherwise we would need to take $R(x) \frac x 2 \Theta(x)$. In other words $E^\mu$ is the number of bits that we need to change, in order to classify $\xi^\mu$ correctly.
\end{example}

\section{Convergence of the annealing process}
In this section we want to discuss the convergence properties of the annealing procedure introduced above. The two major questions are: Does the process converge to an invariant measure, and if so, does this measure have the desired property of favoring dense regions? This question is not addressed in \cite{Baldassi_etal}. However, we feel that it is the first problem that needs to be analyzed. Indeed, if the process does not converge to the desired distribution the question is rather when to stop it than what is the optimal choice of parameters.

We will distinguish two cases: the first is when $\gamma$ in the definition of the measures $Q$ in \eqref{eq:Q} and
$\overline Q$ in \eqref{eq:Qbar} does not depend on time, while the second is, when it does.

Before analyzing these two cases, we will slightly modify the annealing procedure, to make it accessible to the best results that are available for discrete time, see \cite{azencott}. As a matter of fact, we find discrete time slightly more appropriate for computer simulations than the continuous time set-up in e.g. \cite{HolleyStroock}, \cite{GoetzeSA}, or \cite{DeuschelMazza}. To this end, we will study cooling schedules, where the inverse temperature $\beta_n$ is fixed for $T_n$ consecutive steps of the annealing process. Denote by $\nu_n$ the distribution of the annealing process $X_n$ at time
$$
L_n := T_1 + \ldots + T_n.
$$
Note that $\nu_n$ can be computed recursively: If $S_\beta$ denotes the transition matrix of the Metropolis-Hastings chain (see \cite{haeggstroem} or \cite{hastings})  at inverse temperature $\beta$ (see \eqref{SimAn} below), then
$$
\nu_n = \nu_{n-1} S_{\beta_n}^{T_n} \qquad \mbox{where } \nu_0 \mbox{ is a fixed probability measure on }\Sigma^y.
$$
Here, of course, $S_{\beta_n}^{T_n}$ is the $T_n$'th power of the transition matrix $S_{\beta_n}$ (which is constant for the last $T_n$ steps, as described above).

\subsection{Fixed $\gamma$}
If $\gamma$ is fixed it is convenient to split the Simulated Annealing algorithm introduced above into a $\gamma$-dependent part and a $\beta$-dependent part. To this end, let us introduce the following probability measure
$\mu_0$ on $\Sigma^y$:
$$
\mu_0(\{\sigma^a\}):=\frac{\exp\left(\sum_{i=1}^N
\log\cosh(\gamma \sum_{a=1}^y \sigma_i^a)\right)}{\Gamma}
$$
with $\Gamma:= \sum_{\{\tilde \sigma^a\}} \exp\left(\sum_{i=1}^N
\log\cosh(\gamma \sum_{a=1}^y \tilde \sigma_i^a)\right)$.

Next define a transition matrix $\Pi$ on $\Sigma^y$. $\Pi$ will only allow transition from $\{\sigma^a\}$ to  $\{\tilde \sigma^a\}$, if there are exactly one
$\sigma^a, a= 1, \ldots y$ and one $i=1, \ldots N$, such that $\sigma^a_i=-\tilde \sigma^a_i$, and for all other $b$ and $j$ we have $\sigma^b_j=\tilde \sigma^b_j$. In this case, we define
$$
\Pi_\gamma(\{\sigma^a\},\{\tilde\sigma^a\}):=\Pi(\{\sigma^a\},\{\tilde\sigma^a\})
\min\left(1, \frac{\cosh(\gamma \sum_{a=1}^y \tilde \sigma_i^a)}{\cosh(\gamma \sum_{a=1}^y \sigma_i^a)}\right).
$$
For all other configurations $\{\hat\sigma^a\}\neq \{\sigma^a\}$, we have $\Pi_\gamma(\{\sigma^a\},\{\hat\sigma^a\})=0$ and we set
$$\Pi_\gamma(\{\sigma^a\},\{\sigma^a\}):= 1 - \sum_{\{\tilde\sigma^a\}\neq \{\sigma^a\}} \Pi_\gamma(\{\sigma^a\},\{\tilde\sigma^a\}).$$
Note that $\Pi_\gamma$ is nothing but the Metropolis-Hastings algorithm for the measure $\mu_0$ (see \cite{hastings}). In particular, $\Pi_\gamma$ is reversible with respect to the measure $\mu_0$, i.e.
$$
\mu_0(\{\sigma^a\})\Pi_\gamma(\{\sigma^a\},\{\tilde\sigma^a\})=\mu_0(\{\tilde \sigma^a\})\Pi_\gamma(\tilde \{\sigma^a\},\{\sigma^a\}).
$$
Now consider the Metropolis-Hastings chain on $\Sigma^y$ with proposal chain $\Pi_\gamma$ and transition probabilities
\begin{eqnarray}\label{SimAn}
&&S_\beta(\{\sigma^a\},\{\tilde\sigma^a\}):=\\
&&\left\{ \begin{array}{ll} \exp\left(-\beta (\sum_{a=1}^y E(\sigma^a)-\sum_{a=1}^y E(\tilde\sigma^a))^+\right) \Pi_\gamma(\{\sigma^a\},\{\tilde\sigma^a\})& \mbox{if }\{\sigma^a\} \neq\{\tilde\sigma^a\}\\
1- \sum_{\{\hat\sigma^a\}\neq \{\sigma^a\} } \exp\left(-\beta (\sum_{a=1}^y E(\sigma^a)-\sum_{a=1}^y E(\hat\sigma^a))^+\right) \Pi_\gamma(\{\sigma^a\},\{\hat\sigma^a\})& \mbox{if }\{\sigma^a\} =\{\tilde\sigma^a\}
\end{array}
\right.
\nonumber
\end{eqnarray}
Here $(x)^+:= \max\{x,0\}$.
For an appropriate normalizing constant $\hat \Gamma$ this chain has as its invariant measure
\begin{equation}\label{invmeas}
\frac{\exp(-\beta \sum_{a=1}^y E(\sigma^a))}{\hat \Gamma} \mu_0(\{\sigma^a\})=
\overline{Q}(\{\sigma^a\})=:\overline{Q}_\beta(\{\sigma^a\}).
\end{equation}
So indeed for each fixed $\beta>0$, $\gamma>0$ we have found a Metropolis chain for $\overline Q$.

If we now let $\beta=\beta_n$ depend on $n$ in the form described at the beginning of the section, we
arrive at a Simulated Annealing algorithm with piecewise constant temperature.

We will quickly introduce some of Azencott's notation \cite{azencott}.
The invariant measure of $S_{\beta_n}$ is $\overline{Q}_{\beta_n}$.
Recall that we assumed that $\min_{\sigma \in \Sigma} E(\sigma)=0$ and define
\begin{equation}\label{eq:defB}
B:= \min_{\sigma: E(\sigma) \neq 0} E(\sigma).
\end{equation}

Next we bound
$$
||\overline{Q}_{\beta_n}-\overline{Q}_{\beta_{n+1}}||_\infty \le \kappa_1 \exp(-\beta_n B)
$$
for some constant $\kappa_1$.
Indeed such an estimate is true for any difference of Gibbs measures with respect to the same energy function.
To see this let $$\rho_{\beta_n}:= \frac{\exp(-\beta_n U(x))}{Z_n}$$ be a sequence of Gibbs measures with respect to the energy function $U$ on a discrete space of size $K$. We assume $\min_x U(x)=0$ (without loss of generality, otherwise we subtract the minimum from $U$) and $\min_{x: U(x)\neq 0} U(x)= B$.
Let $k:= |\{x: U(x)=0\}|.$
Then, for any $x$ with $U(x)=0$ we simply have
$$
|\rho_{\beta_n}(x)-\rho_{\beta_{n+1}}(x)|= \frac{|\sum_{y: U(y) \neq 0} \exp(-\beta_n U(y))-
\exp(-\beta_{n+1} U(y))|}{Z_n Z_{n+1}}\le  2 (K-k) e^{-\beta_n B},
$$
since $\beta_{n+1}\ge \beta_n$ and $Z_n \ge 1$ for all $n$.
Otherwise, if $U(x)>0$, we trivially can compute
$$
|\rho_{\beta_n}(x)-\rho_{\beta_{n+1}}(x)|\leq e^{-\beta_n U(x)} + e^{-\beta_{n+1} U(x)}
\leq 2 e^{-\beta_n B}.
$$
To describe the spectral gap of $S_{\beta_n}$, for any two $\{\sigma^a\},\{\tau^a\} \in \Sigma^y$
let $\mathcal{P}(\{\sigma^a\},\{\tau^a\})$ be the set of all paths in $\Sigma^y$ from $\{\sigma^a\}$ to $\{\tau^a\}$. For $p \in \mathcal{P}(\{\sigma^a\},\{\tau^a\})$ with vertices $\{\nu^{a,r}\}_{r=0}^M$ define
$$\mathrm{Elev}(p):= \max_{\{\nu^{a,r}\}_{r=0}^M} \sum_{a=1}^y E(\nu^{a,r}).
$$
Moreover define $$H(\{\sigma^a\},\{\tau^a\})= \min_{p \in  \mathcal{P}(\{\sigma^a\},\{\tau^a\})} \mathrm{Elev}(p)
$$
and
\begin{equation}\label{eq:def m}
m:= \max_{\{\sigma^a\},\{\tau^a\}}(H(\{\sigma^a\},\{\tau^a\})-\sum_{a=1}^y E(\sigma^a)-\sum_{a=1}^y E(\tau^a).
\end{equation}
The quantity $m$ is related to the optimal cooling schedule for Simulated Annealing as well as to the spectral gap of the associated Metropolis Hastings algorithm $S_\beta$.
To understand this, define the operator $L_\beta f(\{\sigma^a\})$ for $\{\sigma^a\}\in \Sigma^y$ and $f:\Sigma^y \to \R$ by
$$
L_\beta f(\{\sigma^a\})= \sum_{\{\tau^a\}} f(\{\sigma^a\})-f(\{\tau^a\})S_\beta(\{\sigma^a\},\{\tau^a\}).
$$
Let $\mathcal{E}_\beta$ be the associated Dirichlet form, i.e. for functions $f,g:\Sigma^y \to \R$
\begin{eqnarray*}
\mathcal{E}_\beta (f,g)&:=& -\int f L_\beta g d\overline Q_\beta\\
&=&\frac{1}{2 \hat \Gamma}\sum_{\{\sigma^a\},\{\tau^a\}}(f(\{\sigma^a\})-f(\{\tau^a\})
(g(\{\sigma^a\})-g(\{\tau^a\})\\
&&\qquad \times\exp\left(-\beta(\sum_a E(\sigma^a) \vee \sum_a E(\tau^a))\right)
\mu_0(\{\sigma^a\})\Pi_\gamma(\{\sigma^a\},\{\tau^a\}).
\end{eqnarray*}
Then with
$$\psi(\beta):= \inf \{\mathcal{E}(f,f): ||f||_{L^2(\overline Q_\beta)}=1 \mbox{ and } \int f d\overline Q_\beta=0\}
$$
we have
\begin{proposition}\label{HS}
There are constants $c>0$ and $C < \infty$ such that for all $\beta \ge 0$,
$$
c e^{-\beta m} \le \psi (\beta) \le C e^{-\beta m}.
$$
\end{proposition}
\begin{proof}
See \cite[Theorem 2.1]{HolleyStroock}.

\hfill $\Box$
\end{proof}

But we also have that
$$
\psi(\beta)=1-\lambda_1(\beta)
$$
where $\lambda_1(\beta)$ is the second largest eigenvalue of $S_\beta$, cf. \cite[p.176]{Horn_matrix} or
\cite[(1.2)]{DiaconisStroock}. This establishes the relation of $m$ to the spectral gap of $S_\beta$.

Introduce
$$\vep_n := ||\overline{Q}_{\beta_n}-\nu_n ||_\infty.$$
Then
\begin{theorem}\label{theo:theo1}
$\vep_n$ converges to 0, if
\begin{equation}\label{schedule}
\lim_{n \to \infty} \left[ - \sum_{k=1}^n T_k \exp(-\beta_k m) + n \log \kappa_1\right]=-\infty.
\end{equation}
In particular, we need that $\sum_{k=1}^n T_k \exp(-\beta_k m)\to \infty$.
In this case $\nu_n$ has the same limit as $\overline{Q}_{\beta_n}$ and this is given by a distribution
$\overline{Q}_{\infty}$
on
\begin{equation}\label{eq:defN0}\mathcal{N}_0:=\{\{\sigma^a\}: \sum_{a=1}^y E(\sigma^a)=0\}\end{equation}
such that
$$
\overline Q_\infty (\{\sigma^a\}) \asymp \mu_0(\{\sigma^a\}), \quad \mbox{ if $\{\sigma^a\} \in \mathcal{N}_0$}$$
and $\overline Q_\infty (\{\sigma^a\})=0$, otherwise (here $\asymp$ denotes proportionality). Of course $\overline Q_\infty$ is normalized in such a way that it is a probability measure on $\mathcal{N}_0$.
\end{theorem}
\begin{proof}
The convergence part is basically the content of \cite[Section 7]{azencott}. Note that the computations there are done for a proposal chain $\Pi$ that has the uniform measure as its invariant distribution. However, the proof on p. 231 \cite{azencott} carries over verbatim to our situation.

After that it is easy matter to check that $\overline{Q}_{\beta_n}$ has a limiting distribution
$\overline{Q}_{\infty}$ and that $\overline{Q}_{\infty}$ charges every point in $\mathcal{N}_0$ with a probability proportional to $\mu_0(\{\sigma^a\})$.

\hfill $\Box$
\end{proof}

A choice for $T_k$ where \eqref{schedule} holds is given by
$$
T_k:= \frac{\exp(-m \beta_k)}{C}(\log \kappa_1 +Cb)
$$
for a constant $b>0$, $m$ as given in \eqref{eq:def m}, and $C$ as given in Proposition \ref{HS}.

As Azencott \cite{azencott} points out, in this case
$$
\beta_n \sim \frac{\alpha}{B}\log n+\frac b B n,
$$
for some $\alpha >1$ and $B$ defined as in \eqref{eq:defB},
$$
T_n \sim n^{\frac{\alpha m}B} \exp(\frac{b B}m n),
$$
and $T_n$ and $L_n$ have the same order of magnitude, i.e.\ the algorithm spends most of the time in the lowest temperature band.

One also sees the logarithmic relation between $\beta$ and $T$, i.e.\
$$\beta_n \sim \frac 1 m \log (L_n) \sim \frac 1 m \log (T_n).
$$

We now turn to the question whether this algorithm achieves that typical samples from it have realizations in dense regions of $\Sigma$.
First of all this needs to be defined:
\begin{definition}
Let $\sigma \in \Sigma$ with $E(\sigma)=0$ and let $R>0$ and $k \in \N$. The discrete ball $B_R(\sigma) \subset \Sigma$ with radius $R$, centered in $\sigma$ is called an $(R,k)$-dense with respect to $E$, if there are exactly $k$ global minima $\tau$ of $E$ in
$B_R(\sigma)$. (Without loss of generality all balls considered here and henceforth are Hamming balls.)

\medskip
$\sigma$ is called $R$-isolated, if $\sigma$ is the only global minimum of $E$ in $B_R(\sigma)$.
\end{definition}

The authors in \cite{Baldassi_etal} are not very explicit about a definition of ''dense regions'' and the situation where Replicated Simulated Annealing should be applied. However, from their examples,
they seem to have in mind a situation close to the following caricature:
\begin{situation}\label{generic_sit}
Given $1<a<b<1$ and $\alpha_N \to \infty$, $\delta_N\to \infty$ with
$$\lim_{N\to \infty}\frac{\alpha_N}{\delta_N} =0 \quad \mbox{ as well as   } \lim_{N \to \infty} \frac{\delta_N}N =0,$$
we say that a sequence of energy functions $E_N$ on $\Sigma_N:=\Sigma=\{-1,+1\}^N$ is $(a,b,\alpha_N,\delta_N)$-regular, if it has $b^N$ global minima, if there exists $\sigma \in \Sigma_N$ such that $B_{\alpha_N}(\sigma)$ is $(\alpha_N, a^N)$-dense and such that all the other $b^N-a^N$ minima are $\delta_N$-isolated.
\end{situation}

It is now rather obvious that $\overline Q_\infty (\{\cdot\})$ prefers such dense regions:
\begin{proposition}\label{prop_dense}
Assume we are in the situation described in Situation \ref{generic_sit}. Hence we have a sequence of energy functions that is $(a,b,\alpha_N,\delta_N)$-regular. Then, given $\vep >0$, for any admissible choice of these parameters, there exist $y$, $N_0$ and $\gamma$ such that
$$
\overline Q_\infty (\times_{i=1}^y B_{\alpha_N}(\sigma)):=\overline Q_\infty(\{(\sigma^1, \ldots, \sigma^y):\sigma^a \in B_{\alpha_N}(\sigma) \,\forall a\})\ge 1-\vep
$$
for all $N \ge N_0$.
\end{proposition}
\begin{proof}
Note that $\overline Q_\infty$ has its mass concentrated on the set $\mathcal{N}_0$ (given by equation \eqref{eq:defN0}) and the differences in the mass for the various configurations from this set stem from factor
$$
\mu_0(\{\sigma^a\})=\frac{\exp\left(\sum_{i=1}^N
\log\cosh(\gamma \sum_{a=1}^y \sigma_i^a)\right)}{\Gamma}
$$
Let us just consider the numerators of these weights.

Let $\sigma$ be an
$\delta_N$-isolated minimum of $E_N$. If all $\sigma^1, \ldots, \sigma^y$ are located in $\sigma$, then the numerator of
$\mu_0(\{\sigma^a\})$ equals $\exp\left(N
\log\cosh(\gamma y)\right)$. Otherwise there is at least one $\sigma^a$ that is different from $\sigma$, say in a global minimum $\tau$ of $E_N$. By assumption $d_H(\sigma, \tau)\ge \delta_N$. Thus a configuration that has at least one $\sigma^a=\tau$ has a weight at most $$\exp((N-\delta_N)\log \cosh (\gamma y)+\delta_N \log \cosh ((y-2)\gamma)).$$ Now there are $b^N-a^N$ $\delta_N$ isolated minima.
Hence the sum of the numerators of the probabilities of these isolated minima can be be bounded from above by
$$\exp\left(N \log\cosh(\gamma y)\right)b^N \left(1+ b^{N y}\frac{e^{(N-\delta_N)\log \cosh (\gamma y)+\delta_N \log \cosh ((y-2)\gamma))}} {e^{N \log\cosh(\gamma y)}}\right).
$$
Here $b^N$ is a bound on the number of isolated minima, $\exp\left(N \log\cosh(\gamma y)\right)$ is the weight, when all $\sigma^a$ are identical, $ b^{N y}$ is an upper bound on the number of choices we have, when one $\sigma^a$ equals a given isolated minimum and at least one $\sigma^b$ is different, and finally $e^{(N-\delta_N)\log \cosh (\gamma y)+\delta_N \log \cosh ((y-2)\gamma))}$ is a rough upper bound on the weight in that case.

Note, that we will choose $\gamma$ and $y$ below in such a way that $y \gamma \to \infty$, when $N \to \infty$. We will therefore bound $\log \cosh (y \gamma) \le y \gamma$.
Then the total contribution of the isolated minima becomes at most:
$$
e^{N\gamma y} b^N \left(1+ b^{N y}e^{-2\gamma \delta_N}\right)
$$
If we choose $\gamma \ge \frac{Ny \log b}{2 \delta_N}$ the contribution to the numerator of the probability of the isolated minima will be at most
$2 \exp\left(N \gamma y)\right)b^N$.

On the other hand, for the case that all $\sigma^a$ are in the dense region $B_{\alpha_N}(\sigma)$ we have
$a^{Ny}$ choices. For each of these choices at least $N-y \alpha_N$ of the coordinates of all $\sigma^1, \ldots, \sigma^y$ are identical.
Again, since $y \gamma \to \infty$, when $N \to \infty$, given $\vep'>0$ we may bound $\log \cosh (y \gamma) \ge y \gamma (1-\vep')$.
Thus the overall weight (this is again the numerator of the corresponding probability) of the dense region is at least
$a^{Ny} e^{(N-y\alpha_N) \gamma y(1-\vep')}$.

To compare the two weights, let us see, if we can arrange the parameters in such a way that
$$
a^{Ny} e^{(N-y\alpha_N) \gamma y(1-\vep')} \gg 2 e^{N \gamma y}b^N
$$
(by which we mean that
$$
\frac{a^{Ny} e^{(N-y\alpha_N) \gamma y(1-\vep')}} {2 e^{N \gamma y}b^N} \to \infty
$$
as $N \to \infty$).
Since $\vep'>0$ is fixed and arbitrarily small, we may as well check whether
$$
a^{Ny} e^{(N-y\alpha_N) \gamma y} \gg 2 e^{N \gamma y}b^N$$
which is the case, if and only if
$$
\exp\left(Ny \left(\log a-\frac{y\gamma \alpha_N}N-\frac{\log b}y \right)\right)\gg 1.
$$
To this end, substitute $\gamma = \frac{Ny \log b}{2 \delta_N}$ (and note that indeed $\gamma=\gamma_N \to \infty$ as $N \to \infty$) to obtain for the exponent on the right hand side:
$$
Ny\left(\log a - \frac{y^2 \log b \, \alpha_N}{2 \delta_N} -\frac{\log b}y \right).
$$
Now take $y=\lceil \frac 12\frac{\log b}{\log a}\rceil.$ Since, by assumption $\frac{\alpha_N}{\delta_N} \to 0$ and $y$ does not depend on $N$, also $\frac{y^2 \log b \, \alpha_N}{2 \delta_N}$ converges to $0$. This implies that the exponent will eventually become negative, hence the dense region carries an arbitrarily large mass.

\hfill $\Box$
\end{proof}

\begin{remark}
Reading \cite{Baldassi_etal} carefully, one may get the impression that for them a dense region is one with an exponential number of local minima of $E_N$ (again, the authors in \cite{Baldassi_etal} are not very explicit about this). However, if we are taking the limit $\beta \to \infty$ slowly enough as in a real Simulated Annealing schedule, the local minima that are not global minima will eventually get zero probability and hence are negligible. As a matter of fact, if one works with finite times as in our next section, this is, of course, not true. In this case however, one could equally well study a low temperature Metropolis chain, since most of the time in the annealing schedules is spent in the low temperature region, anyway, as remarked above. For this Metropolis-Hastings chain a result similar to Proposition \ref{prop_dense} can be shown very similarly.
\end{remark}

\subsection{The limit $\gamma \to \infty$.}
The situation where also $\gamma$ depends on time and converges to infinity, when time becomes large, is different to the fixed $\gamma$ situation.
Even though this is not explicitly stated in \cite{Baldassi_etal} it seems
to be the version of the algorithm that the authors in have in mind.
Indeed, as mentioned, they only consider a finite time horizon, in which they, however, increase $\gamma$.

In the situation with $\gamma \to \infty$ we need to modify the considerations of the previous section. Again we will assume that we keep
$\beta_n, \gamma_n$ constant on an interval $T_n \le t \le T_{n+1}-1$. For the algorithm in this fixed time
interval, again, the invariant measure is given by $\overline{Q}_\beta$ with $\beta=\beta_n$ and
$\gamma=\gamma_n$ as given in \eqref{invmeas}. This is the case because during this interval the parameters
of the Metropolis chain do not change. To stress the dependence on both parameters, we will now denote this
measure by $\overline{Q}_{\beta,\gamma}$.

Following the arguments in the previous subsection we now see that there is a constant $\kappa_2$, such that
$$
||\overline{Q}_{\beta_n,\gamma_n}-\overline{Q}_{\beta_{n+1},\gamma_{n+1}}||_\infty \le \kappa_2
e^{-(\beta_n B +\gamma_n B')}.
$$
Here again, $B:= \min_{\sigma: E(\sigma) \neq 0} E(\sigma)$. Analogously, the constant $B'$ is defined as the gap between the maximum of the function
$$H(\{\sigma^a\}):= \sum_{i=1}^N \log \cosh(\sum_{a=1}^y \sigma_i^a) \qquad \mbox{on } \Sigma^y
$$
and its second largest value. Hence
\begin{eqnarray*}
B'&:=& N \log \cosh(\gamma y)-((N-1)\log \cosh(\gamma y)+\log \cosh(\gamma (y-2)))\\
&=& \log \cosh(\gamma y)-\log \cosh (\gamma (y-2)).
\end{eqnarray*}
The maximum of $H$ is realized when we take all $\sigma^a$ identical, while the second term in $B'$ stems from the fact that the we obtain the second largest value of $H$ by changing one $\sigma^a$ in one spin from a maximizing configuration.
Since we will consider the limit $\gamma \to \infty$ we may safely replace $\log \cosh(\gamma y)$ by $\gamma y -\log 2$ and $\log \cosh(\gamma (y-2))$ by $\gamma (y-2) -\log 2$ to obtain $$B' \approx 2 \gamma.$$
To determine how the cooling schedule has to be chosen, we need to estimate the spectral gap of the Metropolis chain. Note that, if we use Proposition \ref{HS} to do so, we run into the problem, that the constants $c$ and $C$ there depend on time, because the energy function does. The solution is, of course, to include this time dependence into the definitions.
Hence for a time $t$ let
$$
F_t(\{\sigma^a\}):= \sum_{a=1}^y E(\sigma^a)- \frac{\gamma_t}{\beta_t}H(\{\sigma^a\}).
$$
Then, we can represent the Simulated Annealing chain, which we will now denote by $S_{\beta, \gamma}$ and which is still given by \eqref{SimAn} (with the only difference that now also $\Pi_\gamma$ depends on time) as a Simulated Annealing algorithm with time-dependent energy function $F_t$, see e.g.\ \cite{ML_SimAN}, \cite{FG_SimAn}. Indeed, in this case we may replace the proposal chain $\Pi_\gamma$ in \eqref{SimAn} to $\Pi$. Here $\Pi$ being in $\{\sigma^a\}$ picks one of $a=1, \ldots y$ and one index $i=1, \ldots, N$ at random and flips $\sigma_i^a$ to $-\sigma_i^a$. Then $S_{\beta, \gamma}$ can be written as
\begin{eqnarray}\label{SimAn2}
&&S_{\beta,\gamma}(\{\sigma^a\},\{\tilde\sigma^a\}):=\\
&&\left\{ \begin{array}{ll} \exp\left(-\beta (F_t(\{\sigma^a\})-F_t(\{\tilde \sigma^a\}))^+\right)
 \Pi(\{\sigma^a\},\{\tilde\sigma^a\})& \mbox{if }\{\sigma^a\} \neq\{\tilde\sigma^a\}\\
1- \sum_{\{\hat\sigma^a\}} \exp\left(-\beta (F_t(\{\sigma^a\})-F_t(\{\hat \sigma^a\}))^+\right)
 \Pi(\{\sigma^a\},\{\hat\sigma^a\})& \mbox{if }\{\sigma^a\} =\{\tilde\sigma^a\}
\end{array}
\right.
\nonumber
\end{eqnarray}
Now we can use results from \cite{ML_SimAN} (cf.\ \cite{ML_GenAl} for related work) to compute the spectral gap $S_{\beta,\gamma}$. In analogy to what we did in the previous subsection, define
$$m_t:= \max_{\{\sigma^a\},\{\tau^a\}}(H_t(\{\sigma^a\},\{\tau^a\})-F_t(\{\sigma^a\})-F_t(\{\tau^a\})
$$
where
$$H_t(\{\sigma^a\},\{\tau^a\})= \min_{p \in  \mathcal{P}(\{\sigma^a\},\{\tau^a\})} \mathrm{Elev}_t(p)
$$
and
$$\mathrm{Elev}_t(p):= \max_{\{\nu^{a,r}\}_{r=0}^M} F_t(\{\nu^a\}).
$$
Again, in analogy to the previous subsection for functions $f,g:\Sigma^y \to \R$ let
the Dirichlet-form $\mathcal{E}_{\beta,\gamma}$ be given by
\begin{eqnarray*}
\mathcal{E}_{t,\beta,\gamma} (f,g)
&=&\frac{1}{2 \hat \Gamma}\sum_{\{\sigma^a\},\{\tau^a\}}(f(\{\sigma^a\})-f(\{\tau^a\})
(g(\{\sigma^a\})-g(\{\tau^a\})\\
&&\qquad \times\exp\left(-\beta(F_t(\{\sigma^a\}) \vee F_t(\{\tau^a\}))\right)
\mu_0(\{\sigma^a\}) \Pi(\{\sigma^a\},\{\tau^a\}).
\end{eqnarray*}
Then for
$$\psi_t(\beta):= \inf \{\mathcal{E}_{t,\beta,\gamma}(f,f): ||f||_{L^2(\overline Q_{\beta,\gamma})}=1 \mbox{ and } \int f d\overline Q_{\beta,\gamma})=0\}
$$
it holds
\begin{proposition}\label{ML}
There are constants $d>0$ and $D < \infty$ such that for all $\beta \ge 0$.
$$
c e^{-\beta m_t} \le \psi_t (\beta) \le C e^{-\beta m_t}.
$$
\end{proposition}
\begin{proof}
This is the content of \cite[Theorem 2.1]{ML_SimAN}.

\hfill $\Box$
\end{proof}
As before Proposition \ref{ML} implies for the second largest eigenvalue $\lambda_1(\beta,\gamma)$ of $S_{\beta,\gamma}$ that
$$
|\lambda_1(\beta_n,\gamma_n)| \le 1-C e^{-\beta_n m_N}.
$$
From linear algebra we therefore obtain that for each $n=1,2, \ldots$ and any probability measure $\nu_0$ on $\Sigma^y$
$$
|| \nu_0 S_{\beta_n,\gamma_n}^{T_n}||_\infty \le \kappa_3 \left(1-C e^{-\beta_n m_n}\right)^{T_n}||\nu_0||_\infty
$$
for some constant $\kappa_3>0$ (cf. the very similar argument for ordinary Simulated Annealing in \cite[(7.8)]{azencott}). Writing again
$$\vep_n := ||\overline{Q}_{\beta_n,\gamma_n}-\nu_n ||_\infty$$
by the recursive structure of the annealing algorithm and the considerations above we obtain the estimate
$$
\vep_n\le \kappa_2 e^{-(\beta_n B +\gamma_n B')}+\kappa_3 \left(1-C e^{-\beta_n m_n}\right)^{T_n} \vep_{n-1}.
$$
Solving this recursive inequality gives
\begin{equation}\label{cooling2}
\vep_n \le \kappa_3^n \prod_{k=1}^n \left(1-C e^{-\beta_k m_k}\right)^{T_k}\left(\sum_{k=1}^n \frac{\kappa_2 e^{-(\beta_n B +\gamma_n B')}}{u_k}+\vep_0\right)
\end{equation}
(cf. \cite[(7.14)]{azencott}). Here
$$
u_k:=\kappa_3^k \prod_{j=1}^k \left(1-C e^{-\beta_j m_j}\right)^{T_j}.
$$
Hence we need to chose our parameters $\beta_n,\gamma_n, T_n$ in such a way that the right hand side converges to zero. In this case we have shown the following theorem.

\begin{theorem}\label{theo:theo2}
If
\begin{equation}\label{condgen}
\kappa_3^n \prod_{k=1}^n \left(1-C e^{-\beta_k m_k}\right)^{T_k}\left(\sum_{k=1}^n \frac{\kappa_2 e^{-(\beta_n B +\gamma_n B')}}{u_k}+\vep_0\right) \to 0
\end{equation}
as $n \to \infty$, the distribution $\nu_n$ of $S_{\beta_n,\gamma_n}$ has the same limit as $\overline Q_{\beta_n,\gamma_n}$ as $n \to \infty$.

Following
\cite[(7.17)]{azencott} a necessary condition for \eqref{condgen} is
$$
\lim_{n \to \infty} \left( -\sum_{k=1}^n T_k e^{-\beta_k m_k}+n \log \kappa_3\right)=-\infty
$$
\end{theorem}

\begin{remark}\label{rem:condgen}
If we are right with the assumption that the authors in \cite{Baldassi_etal} would take $\gamma_n \to \infty$ when time gets large, the result of the theorem is, however, not what the authors in \cite{Baldassi_etal} seem to intend with their introduction of Replicated Simulated Annealing algorithm. Indeed,
when $\beta_n \to \infty$, and $\gamma_n \to \infty$ the measure $\overline Q_{\beta_n,\gamma_n}$ converges to  $\overline Q_{\infty,\infty}$. However, the latter is nothing but the uniform distribution on
$$
\tilde {\mathcal{N}}_0:=\{\{\sigma^1, \ldots, \sigma^y\}: E(\sigma^a)=0 \, \text{ for all } a=1, \ldots y, \mbox{ and }
\sigma^1 = \ldots = \sigma^y\}.
$$
In particular, $\overline Q_{\infty,\infty}$ does not put higher probability on configuration in dense regions of the state space.
\end{remark}

\begin{remark}
Note that for both, Theorem \ref{theo:theo1} and Theorem \ref{theo:theo2}, the cooling schedules have to be chosen very carefully. An anonymous referee remarked that there are simulation algorithms for Gibbs measures that do not use such a cooling strategy as parallel tempering \cite{orlandini}, swapping \cite{GeyerMCMCmaximumLikelihood},\cite{marinari_parisi}, \cite{GeyerThompsonAMCMC}, or equi-energy sampling \cite{KZW}. We are grateful for this remark.

However, there are some issues with these algorithms.
First of all, all of these algorithms simulate Gibbs measures at non-zero temperatures. That means we will obtain an impression of the energy landscapes, but not necessarily convergence towards their global maxima or minima. However, for the simulations in Section 4 this is still an important remark.

The tempering algorithms usually suffer from the deficit that they require computation of partition functions which is as hard as finding the minima or maxima of the energies involved. Swapping circumvents these problems. However, the speed convergence may be a problem (as it is for simulated annealing). In some situations the swapping algorithm converges rapidly (i.e.\ in polynomial time), see e.g.\ \cite{MadrasZhengCW}, \cite{EKLV}, \cite{loewe_vermet_swap}, in others the convergence takes exponentially long, see \cite{BhatnagarRandallTorpidMixingPotts} or \cite{EbbersLoweREM}. The results in \cite{ebbers_loewe_EE} show that equi-energy sampling typically does not overcome the problem of torpid mixing.
\end{remark}
%

In the next section, we empirically study a slightly modified version of the algorithm of Replicated Simulated Annealing, using both synthetic toy datasets and real data bases.


\pgfplotsset{ every non boxed x axis/.append style={x axis line style=-},
     every non boxed y axis/.append style={y axis line style=-}}

\section{Experiments}

Throughout this section, we present various experiments we conducted to empirically study the effectiveness of Replicated Simulated Annealing. While the previous section had an emphasis on theoretical results on the asymptotic behaviour of the algorithm -- which from our point of view is necessary for its introduction -- the current section analyzes its finite time behaviour and the role of the choice of parameters. 

Notice that hence in this simulation section we will necessarily stay closer to the setting in \cite{Baldassi_etal}. Especially, other than in the preceding theoretical section we will not let $\beta$ and $\gamma$ tend to infinity (for the theoretical part this was necessary in order to get convergence results, while it is impossible in practical applications). The precise setting will be described below. We will put an emphasis on studying the effect of the choice of these hyperparameters on the performance of our algorithm, as well as the robustness of the found solutions.

We conduct our experiments using the MNIST dataset, also
described below, and synthetic data.

\subsection{MNIST dataset}

MNIST is a dataset of images depicting digits between 0 and 9. We randomly choose a learning set of 6,000 examples per digit, i.e. these examples are used to calibrate the model. The aim is  to train a classifier to correctly predict which digits are depicted in previously unseen images. This ability of generalization is measured using a test set containing 1,000 examples per digit, distinct from those appearing in the training set. The proportion of correctly classified images in the training set (resp. test set) is called the training accuracy (resp. test accuracy). MNIST images are 28x28 pixels and grey-leveled. As such, they are typically represented by a 784-sized vector of numbers between 0 and 255.

When training a binary (weights can only be -1 or 1) logistic regression classifier on MNIST using Replicated Simulated Annealing, we typically achieve a 88\% accuracy on the test set, which is on par with the performance obtained with continuous weights and gradient descent. Note that when training our models, we use the cross-entropy loss as our energy, which we refer to as the training loss in the following.
In the case of classification with $K=10$ classes, the output of the model associated to an input $x_i$ is a probability vector $y_i=(y_{i,1},\ldots, y_{i,K})$, and the cross-entropy is then 
$$H(t,y) = - \sum_{i=1}^n  \sum_{k=1}^K t_{i,k} \log(y_{i,k}),$$
where $t_i\in\{0,1\}^K$ is the true class of $x_i$. 
Beside, the cross-entropy loss on the test set is referred to as the test loss. We train the networks for a total of 300,000 total iterations, starting from a random configuration. Our models contain a total of $784\cdot 10 = 7840$ parameters, corresponding to a single matrix the input of which is a raw image of 784 dimensions and the output of which is a 10-sized vector where the largest coordinate indicates the associated decision.

\subsection{Effect of the initial and final values of $\beta$}

As mentioned above in our experiments we will always take $\beta$ from a certain bounded range of values $[\beta_i, \beta_f]$ (these bounds will be used in the remaining of this work).
We first explore the influence of the initial and final values for $\beta$. Throughout our experiments, we change the value of $\beta$ from $\beta_i$ to $\beta_f$ following an exponential interpolation where $\beta = \beta_i \left(\beta_f/\beta_i\right)^{it/it_{\max}}$, $it$ being the current number of iterations and $it_{\max}$ the total number of iterations. This choice of interpolation appeared to give the best and most consistent results among the interpolations we tried, including linear and quadratic with various parameters. Note that for this first series of experiments we only train one model ($\gamma = 0$). First in Table~\ref{active_changes_beta}, we indicate the number of active transitions (when a potential flip of a value has been performed). Little surprisingly, we observe that the higher the values of $\beta$, the less likely we perform flips. We observe a range of two orders of magnitude with our selected parameters.

\begin{table}[h]
  \caption{Number of active transitions, during the learning of MNIST, as a function of $\beta_i$ and $\beta_f$.}
  \label{active_changes_beta}
  \begin{center}
    \begin{tabular}{|c||c|c|c|c|c|c|}
      \hline
      $\beta_i$/$\beta_f$ & 10 & 100 & 1,000 & 10,000 & 100,000 & 1,000,000\\
      \hline
      \hline
      1 & 271,733 & 246,793 & 204,681 & 160,424 & 129,138 & 107,549 \\
      \hline
      10 & 251,891 & 219,362 & 172,511 & 122,851 & 92,480 & 74,444 \\
      \hline
      100 & 220,294 & 182,280 & 121,842 & 73,605 & 48,810 & 36,684 \\
      \hline
      1,000 & 171,610 & 122,934 & 60,901 & 24,248 & 13,690 & 10,216\\
      \hline
      10,000 & 125,211 & 76,515 & 25,593 & 7,542 & 5,061 & 4,415 \\
      \hline
    \end{tabular}
  \end{center}
\end{table}

In Table~\ref{train_set_beta}, we depict the corresponding training loss and training accuracy. Interestingly, we observe that the largest values of $\beta$ are not necessarily giving the best results, suggesting that allowing to perform flips that immediately slightly lower the loss can be beneficial in the long run. We also observe that the results do not seem to
be very sensitive of the choice of the initial and final values for $\beta$, as a large range of these values yield a very similar performance. Together with Table~\ref{active_changes_beta}, we can observe that $\beta_i = 100$ and $\beta_f = 100,000$ is a reasonable choice of parameters. This is also confirmed by the results given in Table~\ref{test_set_beta} where we depict the corresponding test loss and test accuracy.

\begin{table}[h]
  \caption{Final MNIST train set loss and corresponding accuracies, as a function of $\beta_i$ and $\beta_f$. Bold numbers correspond to the best results across the table.}
  \label{train_set_beta}
  \footnotesize
  \begin{center}
    \begin{tabular}{|c||c|c|c|c|c|}
      \hline
      $\beta_i$/$\beta_f$ & 
      100 & 1,000 & 10,000 & 100,000 & 1,000,000\\
      \hline
      \hline
      1 & 
      3.49 (77.56\%) & 1.71 (86.08\%) & 1.47 (87.57\%) & 1.43 (87.53\%) & 1.47 (87.35\%)\\
      \hline
      10 & 
      3.54 (76.84\%) & 1.72 (86.10\%) & 1.44 (87.29\%) & 1.44 (87.44\%) & 1.44 (87.41\%)\\
      \hline
      100 & 
      3.32 (77.66\%) & 1.64 (86.34\%) & 1.35 (87.78\%) & \textbf{1.32} (87.56\%) & 1.37 (\textbf{87.64\%})\\
      \hline
      1,000 & 
      3.18 (78.60\%) & 1.59 (86.00\%) & 1.45 (87.25\%) & 1.52 (87.03\%) & 1.50 (86.94\%)\\
      \hline
      10,000 & 
      3.14 (77.57\%) & 1.69 (85.07\%) & 1.61 (86.16\%) & 1.59 (86.95\%) & 1.59 (86.70\%)\\
      \hline
    \end{tabular}
  \end{center}
\end{table}

\begin{table}[h]
  \caption{Final MNIST test set loss and corresponding accuracies, as a function of $\beta_i$ and $\beta_f$. Bold numbers correspond to the best results across the table.}
  \label{test_set_beta}
  \footnotesize
  \begin{center}
    \begin{tabular}{|c||c|c|c|c|c|}
      \hline
      $\beta_i$/$\beta_f$ & 
      100 & 1,000 & 10,000 & 100,000 & 1,000,000\\
      \hline
      \hline
      1 & 
      3.27 (78.29\%) & 1.69 (86.66\%) & 1.51 (87.50\%) & 1.55 (87.46\%) & 1.56 (87.51\%)\\
      \hline
      10 & 
      3.41 (77.45\%) & 1.72 (85.94\%) & 1.50 (87.62\%) & 1.54 (87.34\%) & 1.52 (87.35\%)\\
      \hline
      100 & 
      3.05 (78.84\%) & 1.60 (86.31\%) & 1.51 (87.48\%) & \textbf{1.36 (87.69\%)} & 1.46 (87.46\%)\\
      \hline
      1,000 & 
      3.11 (78.88\%) & 1.68 (86.31\%) & 1.61 (87.11\%) & 1.63 (87.09\%) & 1.65 (87.14\%)\\
      \hline
      10,000 & 
      3.13 (77.81\%) & 1.69 (85.77\%) & 1.65 (86.69\%) & 1.73 (86.55\%) & 1.68 (86.46\%)\\
      \hline
    \end{tabular}
  \end{center}
\end{table}

\subsection{Influence of $\gamma$}

To gain a better understanding of the influence of $\gamma$, in the next series of experiments we reduce the number of training samples to accelerate computations. Namely we use 10,000 arbitrary training samples. We perform 10 runs for each value of $\gamma$, choosing the best values of $\beta_i$ and $\beta_f$ found in the previous section. We plot the error bars (confidence interval at 95\%) for each value of $\gamma$. In Figure~\ref{train_gamma} we depict the evolution of the training accuracy and training loss. In Figure~\ref{test_gamma} the evolution of the test accuracy and test loss, and in Figure~\ref{iterations_gamma} the evolution of the number of active transitions. We observe that $\gamma$ helps in finding better solutions, that is to say solutions with higher accuracies on both the training and the test set. That is only true for a limited range though, as increasing $\gamma$ too much lead to dramatic decrease in overall performance. This is not surprising as a too large $\gamma$ leads to forbid many transitions that would result in reducing the loss. Also this may be seen as being in agreement with the findings of Proposition \ref{prop_dense}, Theorem \ref{condgen} and Remark \ref{rem:condgen} in the previous section (even though there we chose the parameters in such a way to convergence to an invariant measure was guaranteed).

\begin{figure}
  \begin{center}
    \begin{tikzpicture}[ultra thick]
\pgfplotsset{set layers}
\begin{axis}[
    scale only axis,
    axis y line=left,
    xlabel=$\gamma$,
    ymin=0.86,
    ymax = 0.883,
    xmax = 2,
    ylabel style={align=center},
    ylabel={train accuracy $y = 3$ \ref{pgfplots:plot1a}, $y=5$ \ref{pgfplots:plot3a}}
]
\addplot[
    mark=*,
    blue,
    error bars/.cd, y dir=both, y explicit,
] plot coordinates {
( 0.000000, 0.879107) += (0, 0.000702) -= (0,  0.000702)
( 0.100000, 0.879204) += (0, 0.000986) -= (0,  0.000986)
( 0.200000, 0.879727) += (0, 0.000580) -= (0,  0.000580)
( 0.400000, 0.880582) += (0, 0.001108) -= (0,  0.001108)
( 0.800000, 0.880669) += (0, 0.000895) -= (0,  0.000895)
( 1.600000, 0.864686) += (0, 0.001672) -= (0,  0.001672)
( 3.200000, 0.395767) += (0, 0.028931) -= (0,  0.028931)
( 6.400000, 0.105233) += (0, 0.018176) -= (0,  0.018176)
};
\label{pgfplots:plot1a}
\addplot[
    mark=*,
    cyan,
    error bars/.cd, y dir=both, y explicit,
] plot coordinates {
( 0.000000, 0.878636) += (0, 0.000556) -= (0,  0.000556)
( 0.100000, 0.878665) += (0, 0.000638) -= (0,  0.000638)
( 0.200000, 0.880159) += (0, 0.000840) -= (0,  0.000840)
( 0.400000, 0.881875) += (0, 0.000570) -= (0,  0.000570)
( 0.800000, 0.880901) += (0, 0.000576) -= (0,  0.000576)
( 1.600000, 0.860830) += (0, 0.001713) -= (0,  0.001713)
( 3.200000, 0.214829) += (0, 0.024035) -= (0,  0.024035)
( 6.400000, 0.094670) += (0, 0.021598) -= (0,  0.021598)
};
\label{pgfplots:plot3a}
\end{axis}
\begin{axis}[
    scale only axis,
    axis y line=right,
    axis x line=none,
    ymin=-1,
    ymax=20,
    ylabel style={align=center},
    ylabel={train loss $y=3$ \ref{pgfplots:plot2a}, $y=5$ \ref{pgfplots:plot4a}}
]
\addplot[
    mark=*,
    red,
    error bars/.cd, y dir=both, y explicit,
] plot coordinates {
( 0.000000, 1.305628) += (0, 0.008274) -= (0,  0.008274)
( 0.100000, 1.307197) += (0, 0.007944) -= (0,  0.007944)
( 0.200000, 1.311356) += (0, 0.007972) -= (0,  0.007972)
( 0.400000, 1.318854) += (0, 0.012751) -= (0,  0.012751)
( 0.800000, 1.346346) += (0, 0.005196) -= (0,  0.005196)
( 1.600000, 1.683810) += (0, 0.024885) -= (0,  0.024885)
( 3.200000, 13.322294) += (0, 1.188810) -= (0,  1.188810)
( 6.400000, 40.643523) += (0, 4.083131) -= (0,  4.083131)
};
\label{pgfplots:plot2a}
\addplot[
    mark=*,
    orange,
    error bars/.cd, y dir=both, y explicit,
] plot coordinates {
( 0.000000, 1.314056) += (0, 0.007175) -= (0,  0.007175)
( 0.100000, 1.311676) += (0, 0.003481) -= (0,  0.003481)
( 0.200000, 1.305734) += (0, 0.006707) -= (0,  0.006707)
( 0.400000, 1.302314) += (0, 0.007908) -= (0,  0.007908)
( 0.800000, 1.367624) += (0, 0.009402) -= (0,  0.009402)
( 1.600000, 1.795039) += (0, 0.019474) -= (0,  0.019474)
( 3.200000, 22.468377) += (0, 1.864691) -= (0,  1.864691)
( 6.400000, 42.601920) += (0, 2.782936) -= (0,  2.782936)
};
\label{pgfplots:plot4a}
\end{axis}
\end{tikzpicture}
  \end{center}
  \caption{Evolution of the train accuracy and train loss as a function of $\gamma$, and for various values of $y$.}
  \label{train_gamma}
\end{figure}
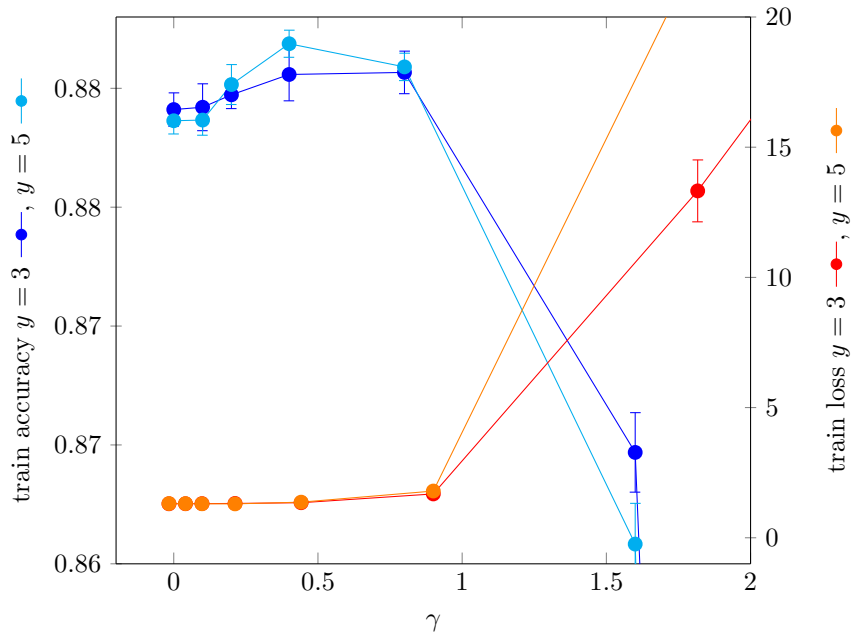

\begin{figure}
  \begin{center}
    \begin{tikzpicture}[ultra thick]
\pgfplotsset{set layers}
\begin{axis}[
    scale only axis,
    axis y line=left,
    xlabel=$\gamma$,
    ymin=0.86,
    ymax=0.883,
    xmax= 2,
    ylabel style={align=center},
    ylabel={test accuracy $y=3$ \ref{pgfplots:plot1b}, $y=5$ \ref{pgfplots:plot3b}}
]
\addplot[
    mark=*,
    blue,
    error bars/.cd, y dir=both, y explicit,
] plot coordinates {
( 0.000000, 0.876153) += (0, 0.000760) -= (0,  0.000760)
( 0.100000, 0.877127) += (0, 0.001441) -= (0,  0.001441)
( 0.200000, 0.877150) += (0, 0.001207) -= (0,  0.001207)
( 0.400000, 0.878077) += (0, 0.001484) -= (0,  0.001484)
( 0.800000, 0.880383) += (0, 0.001071) -= (0,  0.001071)
( 1.600000, 0.866127) += (0, 0.002672) -= (0,  0.002672)
( 3.200000, 0.401423) += (0, 0.030805) -= (0,  0.030805)
( 6.400000, 0.103760) += (0, 0.019258) -= (0,  0.019258)
};
\label{pgfplots:plot1b}
\addplot[
    mark=*,
    cyan,
    error bars/.cd, y dir=both, y explicit,
] plot coordinates {
( 0.000000, 0.876688) += (0, 0.001374) -= (0,  0.001374)
( 0.100000, 0.876752) += (0, 0.001082) -= (0,  0.001082)
( 0.200000, 0.877980) += (0, 0.000913) -= (0,  0.000913)
( 0.400000, 0.880052) += (0, 0.000925) -= (0,  0.000925)
( 0.800000, 0.878604) += (0, 0.000960) -= (0,  0.000960)
( 1.600000, 0.863170) += (0, 0.001711) -= (0,  0.001711)
( 3.200000, 0.218820) += (0, 0.023576) -= (0,  0.023576)
( 6.400000, 0.094756) += (0, 0.023968) -= (0,  0.023968)
};
\label{pgfplots:plot3b}
\end{axis}
\begin{axis}[
    scale only axis,
    axis y line=right,
    axis x line=none,
    ymin=-1,
    ymax=20,
    ylabel style={align=center},
    ylabel={test loss $y=3$ \ref{pgfplots:plot2b}, $y=5$ \ref{pgfplots:plot4b}}
]
\addplot[
    mark=*,
    red,
    error bars/.cd, y dir=both, y explicit,
] plot coordinates {
( 0.000000, 1.440945) += (0, 0.012396) -= (0,  0.012396)
( 0.100000, 1.431905) += (0, 0.014829) -= (0,  0.014829)
( 0.200000, 1.457716) += (0, 0.016440) -= (0,  0.016440)
( 0.400000, 1.457907) += (0, 0.013924) -= (0,  0.013924)
( 0.800000, 1.466174) += (0, 0.017617) -= (0,  0.017617)
( 1.600000, 1.747894) += (0, 0.034602) -= (0,  0.034602)
( 3.200000, 13.047086) += (0, 1.243054) -= (0,  1.243054)
( 6.400000, 40.909024) += (0, 4.099699) -= (0,  4.099699)
};
\label{pgfplots:plot2b}
\addplot[
    mark=*,
    orange,
    error bars/.cd, y dir=both, y explicit,
] plot coordinates {
( 0.000000, 1.450893) += (0, 0.012511) -= (0,  0.012511)
( 0.100000, 1.445800) += (0, 0.010666) -= (0,  0.010666)
( 0.200000, 1.442047) += (0, 0.012832) -= (0,  0.012832)
( 0.400000, 1.439572) += (0, 0.014535) -= (0,  0.014535)
( 0.800000, 1.508703) += (0, 0.012593) -= (0,  0.012593)
( 1.600000, 1.833471) += (0, 0.026771) -= (0,  0.026771)
( 3.200000, 22.264160) += (0, 1.878164) -= (0,  1.878164)
( 6.400000, 43.080360) += (0, 2.886584) -= (0,  2.886584)
};
\label{pgfplots:plot4b}
\end{axis}
\end{tikzpicture}
  \end{center}
  \caption{Evolution of the test accuracy and test loss as a function of $\gamma$, and for various values of $y$.}
  \label{test_gamma}
\end{figure}
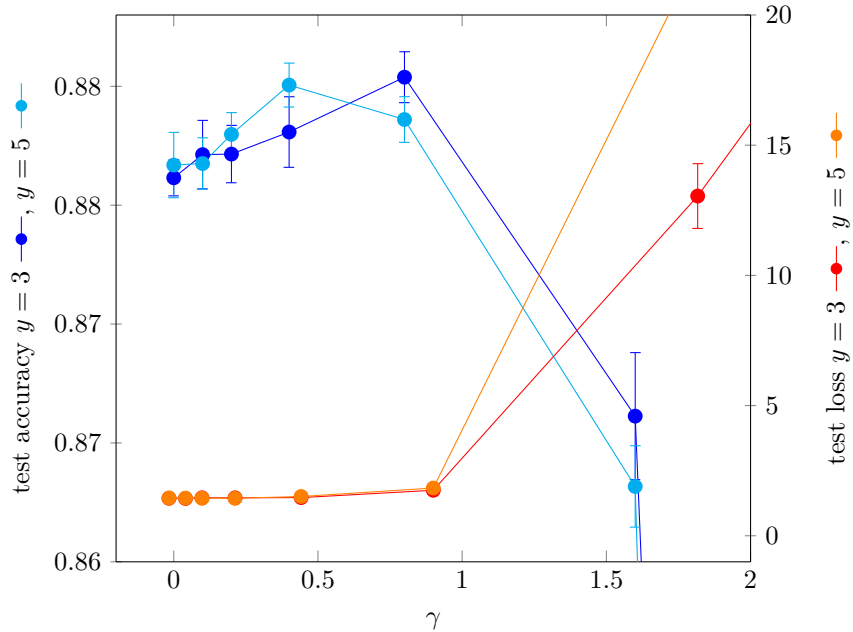

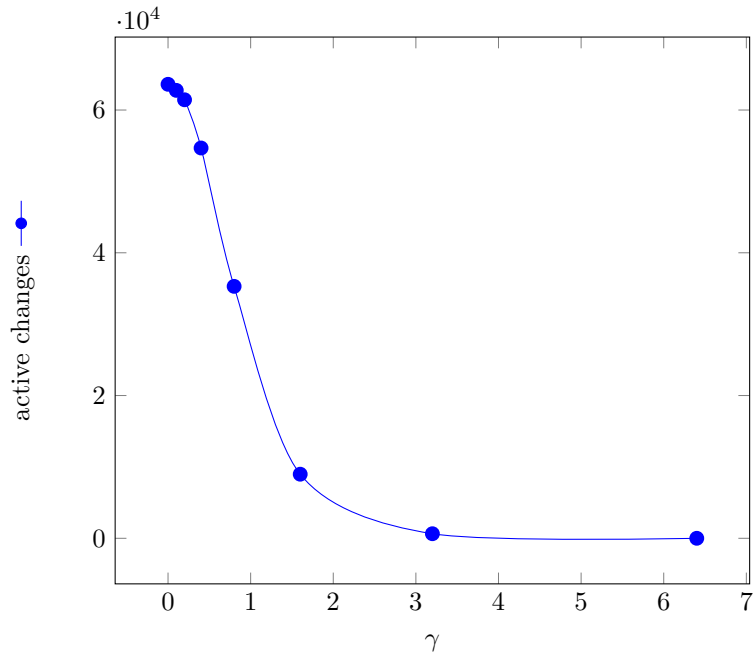
\begin{figure}[h]
  \begin{center}
    \begin{tikzpicture}[ultra thick]
\pgfplotsset{set layers}
\begin{axis}[
    scale only axis,
    xlabel=$\gamma$,
    ylabel style={align=center},
    ylabel={active changes \ref{pgfplots:plot1c}}
]
\addplot[
    smooth,
    mark=*,
    blue,
    error bars/.cd, y dir=both, y explicit,
] plot coordinates {
( 0.000000, 63604.233333) += (0, 235.562395) -= (0,  235.562395)
( 0.100000, 62754.833333) += (0, 232.267374) -= (0,  232.267374)
( 0.200000, 61436.300000) += (0, 196.705869) -= (0,  196.705869)
( 0.400000, 54679.500000) += (0, 235.829507) -= (0,  235.829507)
( 0.800000, 35297.100000) += (0, 186.288589) -= (0,  186.288589)
( 1.600000, 8979.400000) += (0, 115.131056) -= (0,  115.131056)
( 3.200000, 639.933333) += (0, 20.572045) -= (0,  20.572045)
( 6.400000, 1.000000) += (0, 0.449630) -= (0,  0.449630)
};
\label{pgfplots:plot1c}
\end{axis}
\end{tikzpicture}
  \end{center}
  \caption{Evolution of the number of active iterations as a function of $\gamma$, and for various values of $y$.}
  \label{iterations_gamma}
\end{figure}

\subsection{Robustness of trained models}

To study the robustness of trained models, we consider randomly perturbating a proportion $p$ of the weights in the trained models, and evaluating the impact on the test accuracy. We average each point over 1,000 runs of random perturbations, but since it takes a very long time to train the models with MNIST, we always use the same trained models (one for each value of $\gamma$). In Figure~\ref{robustness3}, we depict the results for $y=3$, in Figure~\ref{robustness5} for $y=5$, and in Figure~\ref{robustness7} for $y=7$. In order to add statistically more significant results, we also plot in Figure~\ref{robustnesssynt} results obtained with synthetic data and $y=10$. Synthetic data is created by generating 30 vectors uniformly drawn with repetition from all binary vectors of size 100. In this experiment, we average the results over 1,000 tests for each point. For this additional experiment, we found that the best values are $\beta_i = 0.1$ and $\beta_f = 1,000$.

\begin{figure}
  \begin{center}
    \input{robustness3}
  \end{center}
  \caption{Robustness of trained models on MNIST as a funciton of the proportion of flipped parameters $p$ ($y=3$). Shaded regions around the curves correspond to the confidence interval at 95\%.}
  \label{robustness3}
\end{figure}

\begin{figure}
  \begin{center}
    \input{robustness5}
  \end{center}
  \caption{Robustness of trained models on MNIST as a function of the proportion of flipped parameters $p$ ($y=5$). Shaded regions around the curves correspond to the confidence interval at 95\%.}
  \label{robustness5}
\end{figure}

\begin{figure}
  \begin{center}
    \input{robustness7}
  \end{center}
  \caption{Robustness of trained models on MNIST as a function of the proportion of flipped parameters $p$ ($y=7$). Shaded regions around the curves correspond to the confidence interval at 95\%.}
  \label{robustness7}
\end{figure}

\begin{figure}
  \begin{center}
    \begin{tikzpicture}[ultra thick]
\begin{axis}[
    scale only axis,
    xlabel=$p$,
width = 10.5cm,
height = 6cm,
    ylabel={train accuracy}
]
    \addplot +[mark=none, orange] coordinates {
    (0,1.)
(0.001,0.99816667)
(0.002,0.9961)
(0.005,0.98883333)
(0.01,0.977)
(0.02,0.95593333)
(0.05,0.89969999)
(0.1,0.83906666)
(0.2,0.74346667)
(0.5,0.53853334)
};
    \addlegendentry{$\gamma = 0$}
    \addplot [mark=none, orange, draw=none, name path=upper, forget plot] coordinates {
    (0,1.0)
(0.001,0.9986507501769166)
(0.002,0.9968053086620764)
(0.005,0.990022431243352)
(0.01,0.9786963155524057)
(0.02,0.9582558762053744)
(0.05,0.9030993314507605)
(0.1,0.8432249683354238)
(0.2,0.7484086171326916)
(0.5,0.5441745453146207)
};

    \addplot [mark=none, orange, draw=none, name path=lower, forget plot] coordinates {
    (0,1.0)
(0.001,0.9976825898230834)
(0.002,0.9953946913379236)
(0.005,0.9876442287566479)
(0.01,0.9753036844475943)
(0.02,0.9536107837946256)
(0.05,0.8963006485492395)
(0.1,0.8349083516645761)
(0.2,0.7385247228673084)
(0.5,0.5328921346853792)
};
\addplot [fill=orange!10, forget plot] fill between[of=upper and lower];

    \addplot +[mark=none, blue] coordinates {
    (0,1.)
(0.001,0.99926667)
(0.002,0.99783333)
(0.005,0.99283333)
(0.01,0.98266666)
(0.02,0.96109999)
(0.05,0.90899999)
(0.1,0.84683333)
(0.2,0.7484)
(0.5,0.54036667)
};
    \addlegendentry{$\gamma = 0.4$}
    \addplot [mark=none, blue, draw=none, name path=upper, forget plot] coordinates {
    (0,1.0)
(0.001,0.9995729974184959)
(0.002,0.9983594936386376)
(0.005,0.9937878660206929)
(0.01,0.9841435212062793)
(0.02,0.9632880266228747)
(0.05,0.912254593620639)
(0.1,0.8509087892086381)
(0.2,0.7533104088271887)
(0.5,0.5460062333269226)
};

    \addplot [mark=none, blue, draw=none, name path=lower, forget plot] coordinates {
    (0,1.0)
(0.001,0.9989603425815041)
(0.002,0.9973071663613623)
(0.005,0.991878793979307)
(0.01,0.9811897987937207)
(0.02,0.9589119533771253)
(0.05,0.905745386379361)
(0.1,0.842757870791362)
(0.2,0.7434895911728112)
(0.5,0.5347271066730775)
};
\addplot [fill=blue!10, forget plot] fill between[of=upper and lower];

    \addplot +[mark=none, green] coordinates {
    (0,1.)
(0.001,0.99906667)
(0.002,0.9972)
(0.005,0.99196666)
(0.01,0.98246666)
(0.02,0.96459999)
(0.05,0.91766666)
(0.1,0.85026666)
(0.2,0.75736667)
(0.5,0.53736667)
};
    \addlegendentry{$\gamma = 0.8$}
    \addplot [mark=none, green, draw=none, name path=upper, forget plot] coordinates {
    (0,1.0)
(0.001,0.9994122195184707)
(0.002,0.9977979509973234)
(0.005,0.9929768243927392)
(0.01,0.9839518659555363)
(0.02,0.9666910705519327)
(0.05,0.9207771283571002)
(0.1,0.8543043435194371)
(0.2,0.7622175861505164)
(0.5,0.543008880254702)
};

    \addplot [mark=none, green, draw=none, name path=lower, forget plot] coordinates {
    (0,1.0)
(0.001,0.9987211204815294)
(0.002,0.9966020490026766)
(0.005,0.9909564956072608)
(0.01,0.9809814540444638)
(0.02,0.9625089094480672)
(0.05,0.9145561916428997)
(0.1,0.8462289764805628)
(0.2,0.7525157538494837)
(0.5,0.5317244597452981)
};
\addplot [fill=green!10, forget plot] fill between[of=upper and lower];

    \addplot +[mark=none, black] coordinates {
    (0,1.)
(0.001,0.99592253)
(0.002,0.99398573)
(0.005,0.98178729)
(0.01,0.9667686)
(0.02,0.9346585)
(0.05,0.87689432)
(0.1,0.80924227)
(0.2,0.71811077)
(0.5,0.54176012)
};
    \addlegendentry{$\gamma = 1.6$}
    \addplot [mark=none, black, draw=none, name path=upper, forget plot] coordinates {
    (0,1.0)
(0.001,0.9966436434729531)
(0.002,0.994860667396125)
(0.005,0.9833004729283359)
(0.01,0.9687968945848339)
(0.02,0.9374550106266677)
(0.05,0.88061230915111)
(0.1,0.8136883348877378)
(0.2,0.7232020901353731)
(0.5,0.5473983839304362)
};

    \addplot [mark=none, black, draw=none, name path=lower, forget plot] coordinates {
    (0,1.0)
(0.001,0.9952014165270469)
(0.002,0.993110792603875)
(0.005,0.9802741070716641)
(0.01,0.9647403054151661)
(0.02,0.9318619893733322)
(0.05,0.87317633084889)
(0.1,0.8047962051122622)
(0.2,0.713019449864627)
(0.5,0.5361218560695638)
};
\addplot [fill=black!10, forget plot] fill between[of=upper and lower];

\end{axis}
\end{tikzpicture}
  \end{center}
  \caption{Robustness of trained models with synthetic data as a function of the proportion of flipped parameters $p$ ($y=10$). Shaded regions around the curves correspond to the confidence interval at 95\%.}
  \label{robustnesssynt}
\end{figure}
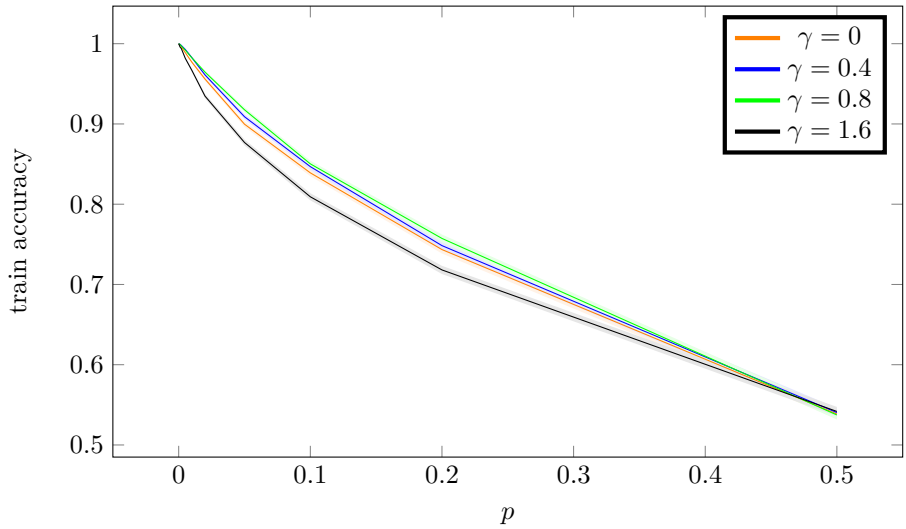

Interestingly, we observe that the most robust models are the ones for a balanced value of $\gamma$, typically 0.8 or 1.6. This is even true for the case of synthetic data, despite the fact all the models start with a perfect accuracy of 100\% when uncorrupted. This is inline with the claims of the authors of~\cite{Baldassi_etal}.

\section{Conclusion}

In this work, we have proposed to mathematically and empirically study the algorithm of Replicated Simulated Annealing, that is used to find good configurations of discrete weights neural networks.
Here the term ``good configurations'' refers to configurations in so called dense regions. We have proposed a definition of such dense regions, which are supposed to yield good generalization properties.
We have given conditions that ensure convergence of the algorithm and discussed its ability to find good configurations in dense robust regions of the search space. We have seen that to do so the parameter $\beta$ always need to be taken to infinity when time becomes large, while the parameter $\gamma$ needs to stay finite.

We also performed experiments using both real datasets and synthetic data to illustrate the role of the choice of the parameters in finite time. Overall, our findings show that Replicated Simulated Annealing is able to find interesting, i.e.\ ''good'', configurations, but that the gain compared to a simple Simulated Annealing is rather small, sometimes even nonexistent in the asymptotic regime, depending on whether one lets $\gamma \to \infty$ or not.


\section*{Acknowledgements}

The research of the second author was funded by the Deutsche Forschungsgemeinschaft (DFG, German Research Foundation) under Germany's Excellence Strategy EXC 2044-390685587, Mathematics M{\"u}nster: Dynamics - Geometry - Structure.

\end{document}